\def\year{2020}\relax
\documentclass[letterpaper]{article} 
\usepackage{aaai20}  
\usepackage{times}  
\usepackage{helvet} 
\usepackage{courier}  
\usepackage[hyphens]{url}  
\usepackage{graphicx} 
\usepackage{graphicx}  
\usepackage{amsmath, amssymb, amsthm, mathtools, listings, inconsolata, enumerate, units,amssymb,dsfont}
\usepackage{bm}
\usepackage{color}
\usepackage{algorithm}
\usepackage{algorithmic}
\usepackage[algo2e]{algorithm2e}
\usepackage{comment}
\usepackage{subfigure}
\usepackage{multirow}
\usepackage{lipsum}

\newcommand\blfootnote[1]{%
	\begingroup
	\renewcommand\thefootnote{}\footnote{#1}%
	\addtocounter{footnote}{-1}%
	\endgroup
}

\def\beq{\begin{equation}\begin{aligned}[b]}
\def\eeq{\end{aligned}\end{equation}}

\newtheorem{theorem}{Theorem}

\newtheorem{proposition}{Proposition}

\DeclareMathOperator*{\argmax}{arg\,max}

\newcommand{\bc}{\bm{c}}

\newcommand{\bx}{\bm{x}}

\newcommand{\bI}{\bm{I}}

\newcommand{\cA}{\mathcal{A}}

\newcommand{\cD}{\mathcal{D}}

\newcommand{\cP}{\mathcal{P}}

\newcommand{\EE}{\mathbb{E}}

\newcommand{\PP}{\mathbb{P}}

\newcommand{\RR}{\mathbb{R}}

\newcommand{\bdelta}{\bm{\delta}}

\newcommand{\btheta}{\bm{\theta}}

\providecommand{\norm}[1]{\|#1\|}

\title{Regularized Training and Tight Certification for Randomized Smoothed Classifier with Provable Robustness}

\author{Huijie Feng\textsuperscript{\rm 1}
, Chunpeng Wu\textsuperscript{\rm 2}, Guoyang Chen\textsuperscript{\rm 2},  Weifeng Zhang\textsuperscript{\rm 2}, Yang Ning\textsuperscript{\rm 1}
\\
\textsuperscript{\rm 1}Cornell University,
\textsuperscript{\rm 2}Alibaba Group Inc.\\
}
\begin{document}
	
	\maketitle
	
	
	\begin{abstract}
		Recently\blfootnote{This work was done when Huijie Feng and Chunpeng Wu were interns at Alibaba.} smoothing deep neural network based classifiers via isotropic Gaussian perturbation is shown to be an effective and scalable way to provide state-of-the-art probabilistic robustness guarantee against $\ell_2$ norm bounded adversarial perturbations. However, how to train a  good base classifier that is accurate and robust when smoothed has not been fully investigated.
		In this work, we derive a new regularized risk, in which the regularizer can adaptively encourage the accuracy and robustness of the smoothed counterpart when training the base classifier. It is computationally efficient and can be implemented in parallel with other empirical defense methods. We discuss how to implement it under both standard (non-adversarial) and adversarial training scheme. At the same time, we also design a new certification algorithm, which can leverage the regularization effect to provide tighter robustness lower bound that holds with high probability. Our extensive experimentation demonstrates the effectiveness of the proposed training and certification approaches on CIFAR-10 and ImageNet datasets. 
		
	\end{abstract}
	
	\section{Introduction}
	Modern machine learning models such as deep neural networks have achieved a great success in a wide range of tasks,
	but are shown to be brittle against \textit{adversarial attacks}. For instance, in image classification small perturbations imperceptible to human eyes may largely deteriorate the performance \cite{szegedy2013intriguing}. 
	Various heuristic approaches are proposed to either \textit{attack} the classifier or \textit{defend} adversarial attacks by making the classifier robust. However, defenses that are empirically observed to be robust to specific types of attacks are later found vulnerable to stronger or adaptive attacks \cite{carlini2017adversarial,athalye2018obfuscated,uesato2018adversarial}. Therefore, achieving provable/certifiable robustness starts to draw attention, in which the goal is to guarantee, deterministically or probabilistically, that no attacks within a certain region will alter the prediction of a  classifier.   
	
	Recently, \textit{randomized smoothing} is shown to be able to provide instance-specific $\ell_2$ robustness guarantees \cite{Lcuyer2018CertifiedRT,Li2018CertifiedAR,cohen2019certified}. Specifically, given a base classifier, the prediction of the smoothed classifier, defined as the most probable prediction over random isotropic Gaussian perturbations, will not change within an $\ell_2$ ball whose radius may vary among different inputs. This guarantee does not require assumptions on the base classifier, and is shown to be one of few methods to provide non-trivial robustness guarantee for large scale classification task like ImageNet. 
	
	Despite recent advances on the theoretical properties of randomized smoothed classifier, how to train a good base classifier that can achieve both good accuracy and robustness when smoothed under this framework has not been fully investigated. The training procedures employed in most previous works did not fully take into account the ultimate goal of achieving high accuracy and robustness when the trained classifier is smoothed.  
	On the other hand, since smoothed classifiers based on neural networks cannot be evaluated exactly (we will discuss the technical details later), 
	in order to provide robustness guarantee under this framework, a certification algorithm is required to give a lower bound of the certified radius for each instance that will hold with high probability. Nevertheless, how to certify the robustness of smoothed classifiers is under-explored as well.

	In this paper, we fill the aforementioned gaps and study how to train and provide robustness certification for randomized smoothed classifier. For training, 
	we derive a regularized risk and discuss how to implement it for training a good base classifier. Specifically, we propose ADRE, an ADaptive Radius Enhancing regularizer, which penalizes examples misclassified  by the \textit{smoothed} classifier while encourages the certified radius of correctly classified examples. 
	This regularizer can be implemented efficiently and applied in parallel with other adversarial defense methods. In particular, we discuss how ADRE regularization can be extended to adversarial training scheme that has been widely employed to improve adversarial robustness \cite{kurakin2016adversarial,madry2017towards,salman2019provably}. 
	At the same time, we introduce T-CERTIFY, a new certification algorithm to provide a tighter lower bound of the certified radius that holds with high probability. This algorithm builds upon and extends previous certification approaches and can further improve the robustness guarantee. We assess the effectiveness of ADRE and T-CERTIFY on CIFAR-10 and ImageNet datasets, and demonstrate that both approaches can improve the $\ell_2$ robustness of randomized smoothed classifier.
	
	
	\section{Related Work and Preliminary}
	
	\subsubsection{Certified adversarial defenses}
	Certified defenses aim to provide robustness guarantee for classifiers. Specifically, for a certain type of attack, we say a classifier is provable/certifiable robust within some region that may depend on the input, if the outputs of the classifier is constant over this region. For the well studied $\ell_p$ norm bounded attacks, a variety of methods based on techniques such as mixed integer linear programing \cite{lomuscio2017approach,fischetti2017deep}, satisfiability modulo theories \cite{katz2017reluplex,ehlers2017formal,huang2017safety}, 
	bounding local or global Lipschitz constant of the neural network \cite{hein2017formal,cisse2017parseval,tsuzuku2018lipschitz,anil2018sorting},
	convex relaxation \cite{wong2017provable,raghunathan2018semidefinite} and many others have been proposed. However, these methods are generally unable to certify large networks, and thus cannot provide meaningful guarantees for tasks like ImageNet classification, mainly due to the intrinsic computational burden or loose relaxation. Compared to these methods, a salient advantage of randomized smoothed classifier is that it circumvents additional assumptions on the base classifier, and thus can fully leverage large expressive neural network to generate a powerful smoothed classifier.  
	
	\subsubsection{Notations and Randomized Smoothed Classifier}
	Let $\cD$ denote the distribution of $(\bx,y) \in \RR^d \times [C]$ where $[C] = \{1,\dotso,C\}$. A soft classification function parameterized by $\btheta$, $F(\bx;\btheta):\RR^d \rightarrow [0,1]^{C}$, maps the input to the probability score for each class $c \in [C]$,  and the corresponding (hard) classifier $f(\bx;\btheta):\bx\rightarrow [C]$ outputs the class label with the highest score. We use $F^c(\bx)$  to denote the probability score with respect to class $c$. For neural network classifiers, the probability scores are typically generated by the softmax function. 
	
	Given a (base) classifier $f$, the smoothed classifier $g$ based on $f$ under isotropic Gaussian perturbation with variance $\sigma^2$ is defined as
	\beq
	g(\bx;f,\sigma) = \argmax_{c\in [C]}G^c(\bx;f,\sigma),
	\eeq 
	where $G^c(\bx;f,\sigma) = \PP(f(\bx+\bdelta) =c)$ is the smoothed probability score and $\bdelta\sim N(0,\sigma^2\bI).$
	Throughout the paper we simplify the notation by omitting the parameter $\btheta$ and/or $\sigma$, and use $f,g$ to denote the base and smoothed classifier, respectively. 
	A nice property of $g$ is that, for any given $\bx$, $g(\bx+\gamma)$ will yield the same prediction for all $\norm{\gamma}_2 \leq R$, where the certified radius $R$ depends on the top probability score $p_A = \max_cG^{c}(\bx)$ and the ``runner up'' score $p_B = \max_{c\neq g(\bx)}G^c(\bx)$ \cite{Lcuyer2018CertifiedRT,Li2018CertifiedAR,cohen2019certified}. Without further assumptions on $f$, the tight radius is
	\beq \label{radius-ori}
	R = \frac\sigma2(\Phi^{-1}(p_A) - \Phi^{-1}(p_B)),
	\eeq 
	where $\Phi^{-1}$ is the quantile function of standard Gaussian distribution \cite{cohen2019certified}. 
	
	\subsubsection{Training the Base Classifier}
	To train the base classifier, the most common approach was applying canonical empirical risk minimization with a single draw of Gaussian noise added on the training samples as a data augmentation procedure \cite{Lcuyer2018CertifiedRT,cohen2019certified}. Stability training that penalizes the difference between the logits from original and Gaussian augmented example was also proposed \cite{Li2018CertifiedAR}. 
	Very recently, adversarial training was applied to significantly improve the certified $\ell_2$ robustness of randomized smoothed classifier \cite{salman2019provably}, where adding multiple Gaussian perturbation for a single training example was also employed. 
	In this paper, we formalize the idea of single and multiple Gaussian augmentation as approximately minimizing a perturbed risk, based on which we derive the proposed ADRE regularized risk. 
	We further adapt adversarial training to our regularized procedure and demonstrate through experiments that ADRE regularizer is also effective in this case.
	\subsubsection{Robustness Certification}
	The robustness radius for a given example under the framework of randomized smoothing requires identifying and evaluating $p_A$ and $p_B$. Unfortunately, for neural network based smoothed classifier, exact evaluation is intractable. In practice, we can only give a lower bound of the certified radius by estimating a lower and upper bound for $p_A$ and $p_B$, denoted by $\underline{p_A}$ and $\overline{p_B}$, respectively.	Simultaneous confidence interval for multinomial distribution \cite{sison1995simultaneous} was applied in \cite{Li2018CertifiedAR}. However, from statistical perspective, without prior knowledge about the true top and ``runner-up'' class, constructing confidence intervals for class probabilities is not sufficient to provide rigorous robustness certification.  
	Another approach named CERTIFY firstly estimates $\underline{p_A}$, and then chooses $\overline{p_B} = 1 - \underline{p_A}$, which can be loose in some cases \cite{cohen2019certified}. In particular, the proposed ADRE regularizer encourages robustness by penalizing the ``runner-up'' probability for correctly classified examples, and thus this approach may not fully express the improved robustness. In contrast, the proposed T-CERTIFY estimate $\underline{p_A}$ and $\overline{p_B}$ separately, and is shown to provide tighter lower bound for the true certified radius.

	While the radius in (\ref{radius-ori}) holds for arbitrary base classifier, under the framework of randomized smoothing we wish to train a base classifier that can consistently make correct predictions under isotropic Gaussian perturbation to achieve high accuracy and large certified radius
	. Consequently, standard empirical risk minimization may not yield a desired base classifier, since the original and perturbed samples can be very different in high dimension, especially when $\sigma$ is large. Instead, consider the following perturbed risk
	\beq \label{risk-general}
	R_{per}(\btheta,\cD,\cP) &= \EE_{\cD \times \cP}\bigg[
	L(F(\bx+\bdelta;\btheta),y)
	\bigg]\\
	&= \EE_{\cD}\bigg[
	\EE_{\cP}\big[
	L(F(\bx+\bdelta;\btheta),y)\big|\bx.y
	\big]
	\bigg],
	\eeq 
	where $\bdelta \sim \cP $ is the perturbation distribution and $L$ is some loss function. Although $\cP$ and $L$ can be arbitrary, in this paper we focus on $\cP \stackrel{d}{=}N(0,\sigma^2\bI)$
	independent of $\cD$ and cross entropy loss $l_{CE}(F(\bx),y) = -\log(F^y(\bx))$. We write for simplicity $R_{per}(\btheta,\cD,\cP) = R_{per}(\btheta)$ without confusion. Intuitively, minimizing (\ref{risk-general}) yields a classifier that has low risk, and thus high accuracy under Gaussian perturbation. 

	\subsubsection{Motivating Adaptive Radius Enhancing Regularization}
	The perturbed risk (\ref{risk-general}) tends to yield a randomized smoothed classifier with high accuracy. However, the tradeoff between robustness and accuracy has been widely observed, both empirically and theoretically \cite{fawzi2018analysis,tsipras2018robustness,zhang2019theoretically}. 
	Meanwhile, although Gaussian augmentation has also been observed to yield a (base) classifier with improved robustness \cite{kannan2018adversarial}, it is not clear whether it will generate a smoothed classifier with large certified robustness. In fact, without additional assumptions on the curvature or complexity of the base classifier, it is difficult to build a direct connection between the base and smoothed classifier . Thus, the resulting base classifier from (\ref{risk-general}) may still be suboptimal regarding robustness when smoothed.
	
	It is clear that for any given input $\bx$, the certified radius directly depends on the top and ``runner-up'' probability score of the smoothed classifier. Notice that 
	for any fixed input $\bx$ a certified radius exists no matter $g$ makes a correct prediction or not. However, while a large radius when $\bx$ is correctly predicted is desired, a misclassified $\bx$ with large radius is detrimental. This motivates the following measure
	\beq \label{eq-reg}
	R_{adre}(\btheta)
	= \EE_{\cD}\bigg [\underbrace{
		L'(G(\bx;\btheta),\argmax_{c\neq y}G^c(\bx;\btheta))}_{(\spadesuit)}
	\bigg],
	\eeq
	where $L'$ is some loss function. To interpret this, we consider two cases
	\begin{itemize}
		\item when $g$ makes a correct prediction, $\argmax_{c\neq y}G^c(\bx;\btheta)$ is the ``runner-up'' class, and in this case ($\spadesuit$) serves as a measure of robustness for the smoothed classifier, where a larger value suggests a higher robustness.
		\item when $g$ makes a wrong prediction, $\argmax_{c\neq y}G^c(\bx;\btheta)$ is top class, and in this case ($\spadesuit$) corresponds to the radius of a misclassified example, where a larger value indicates a smaller radius. 
		

	\end{itemize}
	Therefore, we can think of $R_{adre}$ as a balanced measure between accuracy and robustness for the smoothed classifier $g$.
	For concreteness, in this paper we also choose $L'$ as the cross entropy loss. Following this, we propose ADRE, an ADaptive Radius Enhancing regularized risk 
	\beq \label{eq-main_risk}
	R_{reg}(\btheta)
	= R_{per}(\btheta) - \lambda R_{adre}(\btheta),
	\eeq
	where $\lambda$ is a hyper-parameter. Here the first component $R_{per}$ corresponds to the classification accuracy of the base classifier under perturbation. For the second component, we use $R_{adre}$ as a regularization term that adaptively encourages the certified radius and accuracy for the smoothed counterpart of the trained base classifier. We call the training procedure based on (\ref{eq-main_risk}) as $\text{ADRE}_{\text{REG}}$.
	\subsubsection{Connection to Large Margin Training}
	The goal of achieving large certified radius for correctly classified example is closely related to the objective of obtaining large margin classifier.  
	Notice that $R = \frac\sigma2(\Phi^{-1}(p_A) - \Phi^{-1}(p_B)) \geq  \phi\cdot\frac{\sigma}2(p_A - p_B),$ where $\phi > 0$ is the lower bound of the derivative of $\Phi^{-1}$. From (\ref{eq-reg}) we can see that $R_{adre}$ acts a similar role as promoting $G^y(\bx) - \max_{c\neq y}G^c(\bx)$, which is equivalent to $p_A -p_B$ when the smoothed classifier correctly classify $\bx$. Therefore, the proposed ADRE regularizer can be treated as a large margin regularizer under the framework of randomized smoothing. Different from directly maximizing the margin of the trained classifier such as in \cite{ding2018max,elsayed2018large}, we exploit $R_{adre}$ that is tailored to randomized smoothed classifiers to guide the trained base classifier in the direction of higher robustness when smoothed.

	\subsubsection{Implementation}
	Given training samples $\{(\bx_i,y_i)\}_{i = 1}^n$, in practice our objective naturally becomes to minimize 
	\beq \label{loss-initial}
	\frac1n \sum_{i =1}^n
	L_i - \lambda P_i,
	\eeq 
	where $L_i =  \EE_{\cP}[l_{CE}(F(\bx_i+\bdelta;\btheta),y_i)]$ and $P_i = l_{CE}(G(\bx_i;\btheta),\argmax_{c\neq y} G^c(\bx_i;\btheta))$. 
	
	However, for a neural network base classifier, it is  intractable to evaluate both $L_i$ and $G$ exactly, and thus we will approximate both terms during training. Given a training pair $(\bx',y')$, for the first term we use the unbiased estimator
	\beq \label{plug-in-0}
	\hat{L}(\bx',y';\btheta) = \frac 1k \sum_{j = 1}^k l_{CE}(F(\bx'+\bdelta_{j};\btheta),y').
	\eeq 
	For the second term, we will substitute $G$ by 
	\beq \label{plug-in}
	\hat{G}(\bx';\btheta) = \frac1k\sum_{j=1}^kF(\bx' + \bdelta_{j};\btheta).
	\eeq 
	Essentially, for both terms we sample i.i.d Gaussian perturbations and substitute the conditional expected loss and the smoothed probability score by finite sample estimators. Note that for $G$, we average over a finite sample of base classifier probability scores $F$ under perturbation instead of employing the fraction of counts, defined as
	\beq\label{esti-unbiased}
	\frac 1k \sum_{j = 1}^k\big(\mathds{1}\{f(\bx' + \bdelta_{ij}) = c\}\big)_{c = 1}^C \in [0,1]^C,
	\eeq
	where $\mathds{1}$ is the indicator function.
	Although (\ref{esti-unbiased}) is an unbiased estimator for $G(\bx')$, due to computational constraint, in practice $k$ cannot be too large, which is problematic both statistically and numerically, especially when the number of classes is large.
	Instead, by applying (\ref{plug-in}) we implicitly conduct smoothing when estimating $G(\bx')$. 

	We can also justify (\ref{plug-in}) following $G(\bx') = (\EE[\mathds{1}(f(\bx'+\bdelta) = c)])_{c = 1}^C \approx\EE [F(\bx'+\delta)]$.
	
	The detailed training procedure is described in Algorithm \ref{algo-train}. Notice that we use the same set of perturbations in both $l_{per}$ and $\hat{G}$. Empirically, we find this saves half of forward pass computation without sacrificing accuracy and robustness compared to the case where two different sets of perturbations are applied. Our implementation of $\text{ADRE}_{\text{REG}}$ also unifies and generalizes different Gaussian data augmentation techniques applied in previous works when $\lambda = 0$ \cite{Lcuyer2018CertifiedRT,cohen2019certified,salman2019provably}. We also note that Algorithm \ref{algo-train} unifies the adversarial training scheme which will be discussed later.
	
	\begin{algorithm}[t]
		\SetKwInOut{Input}{input}
		\SetKwInOut{Par}{parameter}
		\SetKwInOut{Init}{initialize}
		\Input{Training sample $\cD_N$}
		\Par{variance $\sigma > 0$; tuning parameter $\lambda \geq 0$; number of perturbations $k > 0$;\\\# for adversarial training\\	attack steps $M$; step size $\alpha$; radius $\epsilon$;}
		\For{each epoch}{
			
			
			\For{each minibatch $\{(\bx_i,y_i)\}_{i\in [B]} \subset \cD_N$}{

				$\bdelta_{ij} \stackrel{i.i.d}{\sim} N(0,\sigma^2),j \in [k]$
				
				
				\If{ adversarial training}{
				
				\For{$m = 1,\dotso,M$}{
					$\hat{G}_{i} \leftarrow \frac1k\sum_{j=1}^kF(\bx_i + \bdelta_{ij};\btheta)$
					
					$\bx_i \leftarrow$ PGD-step($\bx_i, \hat{G}_{i}, \alpha,\epsilon$)
				}
			}
				$l^{(i)}_{per} \leftarrow \frac1k\sum_{j=1}^k Loss(F(\bx_i+\bdelta_{ij};\btheta),y_i)$ 
				
				$\hat{G}_{i} \leftarrow \frac1k\sum_{j=1}^kF(\bx_i + \bdelta_{ij};\btheta)$ 
				
				{$\hat{y}_i \leftarrow \argmax_{c\neq y_i}\hat{G}_{i}^c$
				}
				
				$l_{adre}^{(i)} \leftarrow Loss(\hat{G}_{i},\hat{y_i})$ 
			}
			$\nabla L = \nabla  
			\frac1B\sum_{i=1}^B \big \{l_{per}^{(i)} - \lambda l_{adre}^{(i)}
			\big\}$ 
			
			$\btheta \leftarrow \text{Step} (\btheta,\nabla L )$ \textit{\#update using proper optimizer}
		}   
		\caption{ADRE regularized Training}\label{algo-train}
	\end{algorithm}
	
	\subsubsection{Alternative Formulations}
	One may consider directly balancing off accuracy and robustness based on the following objective
	\beq \label{eq-ori-radius}
	&\EE_{\cD}\bigg [
	\EE[l_{CE}(F(\bx+\bdelta;\btheta),y)|\bx,y]\\ &- \lambda'
	\big (
	\Phi^{-1}(\max_cG^c(\bx;\btheta) )- 
	\Phi^{-1}(\max_{c\neq g(\bx;\btheta)}G^c(\bx;\btheta))
	\big )
	\bigg],
	\eeq 
	where the first part stays the same, but the second part corresponds to the expected certified radius. Although this looks somewhat natural, empirically we observe that minimizing this objective with plug-in approximation (\ref{plug-in}) is not stable and may converge to bad local minima, especially when $\lambda$ is relatively large. This is reasonable since the second part of (\ref{eq-ori-radius}) does not involve the correct label, and a classifier that consistently makes wrong prediction with high confidence can have low risk. Therefore, minimizing this objective can easily converge to bad local minima with such property. We speculate that more careful initialization may be required to yield desired base classifier in this case. 
	
	\subsubsection{Regularized Smoothed Adversarial Training}
	
	Adversarial training has been widely used to boost the robustness of classifiers, and is arguably the most effective type of empirical defense method against adversarial attacks \cite{kurakin2016adversarial,madry2017towards}. Generally speaking, the objective can be formulated as minimizing the worst case risk over an adversarial region with strength $\epsilon$, denoted by $S_{\epsilon}$
	\beq 
	\EE_{\cD}\bigg [\max_{\bx' \in S_{\epsilon}(\bx)}L(F(\bx';\btheta),y) 
	\bigg ].
	\eeq 
	While adversarial training is typically used to improve empirical robustness of classifiers, it is also recently found helpful to improve provable robustness for smoothed classifier \cite{salman2019provably}. 
	
	In this section, we describe an $\ell_2$ attack scheme based on ADRE regularization, which can be incorporated into training for obtaining robust smoothed classifier. Formally, given $(\bx^0,y)$ we seek for an adversarial example
	\beq \label{eq-attack}
	&\tilde{\bx} = \max_{\norm{\bx^0 - \bx'}_2\leq \epsilon} 
	\bigg \{
	l_{CE}(G(\bx';\btheta),y)\\
	&~~~~~~~~~~~ - \lambda l_{CE}(G(\bx';\btheta),\argmax_{c\neq y}G^c(\bx';\btheta))
	\bigg\}.
	\eeq 
	To be specific, instead of maximizing the standard cross entropy loss of smoothed classifier $ l_{CE}(G(\bx;\btheta),y)$, we maximize it together with ADRE regularization. To interpret this, when we maximize over $ l_{CE}(G(\bx;\btheta),y)$, we generate an adversarial example with respect to the smoothed classifier that leads to high loss and thus wrong prediction. In our scenario, however, $\tilde{\bx}$ tends to be either 1) correctly classified but non-robust or 2) misclassified, potentially by a large margin. Therefore, the proposed attack is more versatile under the framework of randomized smoothing, and potentially leads to a smoothed classifier with a better balance between accuracy and robustness, when adversarial training based on this attack is employed. The proposed attack is an extension of the SMOOTHADV attack  \cite{salman2019provably} when (\ref{eq-attack}) is implemented with plug-in estimate (\ref{plug-in}). We also note that similar to SMOOTHADV we use $ l_{CE}(G(\bx;\btheta),y)$ in (\ref{eq-attack}) instead of $\EE_\delta(F(\bx+\bdelta;\btheta))$, where the latter one was found to be ineffective in practice.

	 Since exact evaluation of the above maximization problem is intractable, we will follow the widely used iterative first-order methods. For concreteness, in this paper we focus on non-targeted $\ell_2$ projected gradient descent (PGD) attack \cite{madry2017towards}, but other approaches can be applied as well. Specifically, we approximate the inner maximizer by iteratively solving
	\beq \label{pgd-imple}
	&\bx^{t+1} = \cP_{2,\epsilon}\bigg(\bx^{t} + \alpha\cdot 
		\nabla \{l_{CE}(\hat{G}(\bx^{t};\btheta),y)\\
		&~~~~~~~~~~~- \lambda l_{CE}(\hat{G}(\bx^{t};\btheta),\argmax_{c\neq y}\hat{G}^c(\bx^{t};\btheta))\}
	\bigg).
	\eeq 
	where $\cP_{2,\epsilon}$ is the projection operator into an $\ell_2$ ball with radius $\epsilon$ and $\alpha$ is the step size. 
	
	The detailed implementation of the proposed adversarial training based on the above PGD attack, referred as $\text{ADRE}_{\text{ADV}}$, is described in Algorithm \ref{algo-train} where the helper function PGD($\bx,\hat{G},\alpha,\epsilon$) runs a single step of PGD iteration (\ref{pgd-imple}). We also reuse the same set of noise samples for each training example at each PGD iteration to stabilize the attack, as suggested in \cite{salman2019provably}.

	\section{Robustness Certification}
	Certifying the robustness radius of a smoothed classifier $g$ for a given input $\bx$ requires evaluating $g$ exactly for $p_A$ and $p_B$. In practice, we may only estimate a lower bound $\underline{p_A}$ and an upper bound  $\overline{p_B}$ that hold with high probability. 	
	In this section, we propose a Monte Carlo algorithm that guarantees a lower bound of the true certified robustness that holds with probability greater than $1 - \alpha$, where $\alpha$ is a pre-specified significance level. This method independently estimates $\underline{p_A}, \overline{p_B}$ and thus can leverage the regularized smoothed classifier to provide tighter robustness guarantee. 
	
	We now describe the certification procedure. Given a base classifier $f$ and input $\bx$, we firstly sample $\bdelta_i \stackrel{i.i.d}{\sim} N(0,\sigma^2 \bI)~\forall~i \in [n]$ and then evaluate each $f(\bx + \bdelta_i)$.
	Suppose we get a sequence of ordered counts $\hat{N}_{R_1} \geq \hat{N}_{R_2} \geq  \dotso \geq \hat{N}_{R_C}$, where each $R_i \in [C]$ is an ordered class label.  
	For a given significance level $\alpha$ and $\alpha' \in [0,\alpha]$, suppose for now $R_1 = g(\bx)$. Consider 
		\beq\label{certi-formula}
		&\underline{p_A} = \sup \bigg \{
		p \big|\PP(Bin(n,p) \geq \hat{N}_{R_1}) \leq  \alpha'
		\bigg \},\\
		& \overline{p_B} = \inf \bigg \{
		p \big | \sum_{j = 2}^{C} \PP(Bin(n,p) \leq \hat{N}_{R_j}) \leq \alpha - \alpha'
		\bigg \},
		\eeq

	where the probability is over the binomial random variable $Bin(n,p)$ with $n$ number of trials and success probability $p$. The lower bound of the top probability score $\underline{p_A}$ is given by the classic Clopper–-Pearson method \cite{clopper1934use} with one-sided significance level $\alpha'$. For the upper bound, we generalize the Clopper–-Pearson method to construct a one-sided confidence interval for $p_B$ with significance level $\alpha - \alpha'$, where $\overline{p_B}$ is defined as its boundary point.  
	
	\begin{proposition}\label{prop-certify}
		Following the certification procedure described above. For any fix $\bx$, if the $R_1 = g(\bx)$ then with probability greater than $1-\alpha$, $g(\bx + \gamma) = R_1~\forall~\norm{\gamma}_2 \leq \frac\sigma2(\Phi^{-1}(\underline{p_A)} - \Phi^{-1}(\overline{p_B}))$.
	\end{proposition}

	\begin{proof}[Proof Sketch] Without loss of generality, suppose the top label $R_1 = 1$. Write
	$G(\bx) = (p_1,p_2,\dotso,p_C)$. 
	It suffices to show that 
	\beq
	\PP(\underline{p_A} > p_1 \cup \overline{p_B} < \max_{c\neq 1}p_c) \leq \alpha.
	\eeq 
	Based on the definitions in (\ref{certi-formula}), we know 
	$
	\PP(\underline{p_A} > p_1) \leq \alpha'. 
	$
	On the other hand, write $\alpha_c = \PP(Bin(n,\overline{p_B}) \leq \hat{N}_{c})$, we know $\sum_{c = 2}^C\alpha_c \leq \alpha - \alpha'$ and therefore
	$$
	\PP(\overline{p_B} < \max_{c\neq 1}p_c) \leq \sum_{c = 2}^C\PP(\overline{p_B} < p_c) \sum_{c = 2}^C\alpha_c \leq \alpha - \alpha'.
	$$ This completes the proof by applying a union bound. 
	\end{proof}
	Proposition \ref{prop-certify} shows that, if we have knowledge about the top class then $\underline{p_A}, \overline{p_B}$ are proper bounds, and thus we can estimate a lower bound for the certified radius that holds with probability greater than $1 - \alpha$. 
	To obtain a tighter lower bound, we may maximize the radius $\frac\sigma2(\Phi^{-1}(\underline{p_A)} - \Phi^{-1}(\overline{p_B}))$ over $\alpha' \in [0,\alpha]$. 
	For practical implementation in which the top class is unknown, we propose T-CERTIFY, which extends CERTIFY \cite{cohen2019certified} to provide a tighter certified robustness for a given input by estimating $\underline{p_A}, \overline{p_B}$ separately and searching over a grid of $\alpha'$s. The algorithm is as follows.
	\begin{algorithm}[ht]
		\SetKwInOut{Input}{input}
		\SetKwInOut{Par}{parameter}
		\Input{base classifier $f$, input $\bx$}
		\Par{variance $\sigma$, size $n_0,n > 0$, significance $\alpha$, grid $\cA$}
		
		$\bc_0 \leftarrow $ SampleUnderNoise($f,\bx,n_0,\sigma$)\;	
		
		
		
		$R_1 \leftarrow$ top index in $\bc_0$\;
		
		$\bc \leftarrow $ SampleUnderNoise($f,\bx,n,\sigma$)\;	
		
		
		\For{$\alpha' \in \cA$}{
			$\underline{p_A} \leftarrow $LowerConfBound($\bc[R_1], n, 1 - \alpha'$) \;
			
			$\overline{p_B} \leftarrow$ UpperConfBound($\bc[-R_1], n, 1 - (\alpha - \alpha')$) \;
			
			\uIf{$\underline{p_A} > 0.5$}{
				$r_{\alpha'} \leftarrow \frac\sigma2(\Phi^{-1}(\underline{p_A)} - \Phi^{-1}(\overline{p_B}))$.\;
			}
			\Else{$r_{\alpha'} \leftarrow 0$.\;
			}
		}
		\textbf{if} $\max_{\alpha'\in \cA} r_{\alpha'} > 0$ \textbf{return} $(R_1,\max_{\alpha'\in \cA} r_{\alpha'})$ \textbf{else return} ABSTAIN
		\caption{T-Certify}\label{algo-certify}
	\end{algorithm}
	
	Here SampleUnderNoise($f,\bx,n,\sigma$) samples the noise $\bdelta'_i~\forall~i\in[n]$, evaluate $f(\bx + \delta'_i)$ and get counts for each class. Function LowerConfBound($c[R_1],n,1 - \alpha$) calculate $\underline{p_A}$ following (\ref{certi-formula}) based on the Clopper–-Pearson confidence interval \cite{clopper1934use}, and similarly for UpperConfBound.  Similar to CERTIFY, T-CERTIFY abstains from making a prediction when the lower bound at significance level $\alpha'$ is no larger than a half, which guarantees the correctness of the top class prediction. 
		\begin{table*}[t]
			\caption{Certified top-1 accuracy on CIFAR-10 and ImageNet at various radii.  }\label{table-main}
			\tiny
			\begin{center}
				\begin{tabular}{ c| c | l | c c c c c c c c c c c c c }
					\hline
					&Method&$\ell_2$ Radius& $0.0$& $0.25$& $0.5$& $0.75$& $1.0$& $1.25$& $1.5$& $1.75$& $2.0$& $2.25$& $2.5$& $2.75$& $3.0$\\
					\hline
					\parbox[t]{2mm}{\multirow{16}{*}{\rotatebox[origin=c]{90}{CIFAR-10}}}&\parbox[t]{2mm}{\multirow{4}{*}{\rotatebox[origin=c]{90}{Basic}}}&$\sigma = 0.12$ & 0.81 & 0.59 & 0.00 & 0.00 & 0.00 & 0.00 & 0.00 & 0.00 & 0.00 & 0.00 & 0.00 & 0.00 & 0.00\\
					&&$\sigma = 0.25$ & 0.75 & 0.60 & 0.43 & 0.27 & 0.00 & 0.00 & 0.00 & 0.00 & 0.00 & 0.00 & 0.00 & 0.00 & 0.00\\
					&&$\sigma = 0.50$ & 0.65 & 0.55 & 0.41 & 0.32 & 0.23 & 0.15 & 0.09 & 0.05 & 0.00 & 0.00 & 0.00 & 0.00 & 0.00\\
					&&$\sigma = 1.00$ & 0.47 & 0.39 & 0.34 & 0.28 & 0.22 & 0.17 & 0.14 & 0.12 & 0.10 & 0.08 & 0.05 & 0.04 & 0.02\\
					\cline{2-16}
					&\parbox[t]{2mm}{\multirow{12}{*}{\rotatebox[origin=c]{90}{$\text{ADRE}_{\text{REG}}$}}}&$\sigma = 0.12,\lambda = 0.1$ & 0.83 & 0.65 & 0.00 & 0.00 & 0.00 & 0.00 & 0.00 & 0.00 & 0.00 & 0.00 & 0.00 & 0.00 & 0.00\\
					&&$\sigma = 0.12,\lambda = 0.2$ & \textbf{0.85} & 0.67 & 0.00 & 0.00 & 0.00 & 0.00 & 0.00 & 0.00 & 0.00 & 0.00 & 0.00 & 0.00 & 0.00\\
					&&$\sigma = 0.12,\lambda = 0.3$ & 0.83 & \textbf{0.68} & 0.00 & 0.00 & 0.00 & 0.00 & 0.00 & 0.00 & 0.00 & 0.00 & 0.00 & 0.00 & 0.00\\
					&&$\sigma = 0.25,\lambda = 0.1$ & 0.78 & 0.64 & \textbf{0.50} & 0.34 & 0.00 & 0.00 & 0.00 & 0.00 & 0.00 & 0.00 & 0.00 & 0.00 & 0.00\\
					&&$\sigma = 0.25,\lambda = 0.2$ & 0.74 & 0.60 & 0.48 & 0.35 & 0.00 & 0.00 & 0.00 & 0.00 & 0.00 & 0.00 & 0.00 & 0.00 & 0.00\\
					&&$\sigma = 0.25,\lambda = 0.3$ & 0.73 & 0.62 & 0.49 & 0.37 & 0.00 & 0.00 & 0.00 & 0.00 & 0.00 & 0.00 & 0.00 & 0.00 & 0.00\\
					&&$\sigma = 0.50,\lambda = 0.1$ & 0.67 & 0.57 & 0.48 & 0.38 & \textbf{0.30} & 0.23 & 0.17 & 0.11 & 0.00 & 0.00 & 0.00 & 0.00 & 0.00\\
					&&$\sigma = 0.50,\lambda = 0.2$ & 0.65 & 0.57 & 0.47 & 0.35 & 0.27 & 0.20 & 0.13 & 0.09 & 0.00 & 0.00 & 0.00 & 0.00 & 0.00\\
					&&$\sigma = 0.50,\lambda = 0.3$ & 0.64 & 0.55 & 0.46 & \textbf{0.38} & 0.30 & \textbf{0.23} & \textbf{0.17} & 0.11 & 0.00 & 0.00 & 0.00 & 0.00 & 0.00\\
					&&$\sigma = 1.00,\lambda = 0.1$ & 0.49 & 0.43 & 0.36 & 0.29 & 0.22 & 0.19 & 0.15 & 0.13 & 0.11 & 0.08 & 0.05 & 0.03 & 0.02\\
					&&$\sigma = 1.00,\lambda = 0.2$ & 0.48 & 0.41 & 0.35 & 0.28 & 0.22 & 0.18 & 0.16 & 0.14 & 0.11 & \textbf{0.09} & 0.06 & 0.05 & 0.02\\
					&&$\sigma = 1.00,\lambda = 0.3$ & 0.47 & 0.39 & 0.33 & 0.29 & 0.24 & 0.20 & 0.17 & \textbf{0.14} & \textbf{0.12} & 0.09 & \textbf{0.07} & \textbf{0.05} & \textbf{0.03}\\
					\hline
					\hline
					\parbox[t]{2mm}{\multirow{9}{*}{\rotatebox[origin=c]{90}{ImageNet}}}&\parbox[t]{2mm}{\multirow{3}{*}{\rotatebox[origin=c]{90}{Basic}}}&$\sigma = 0.25$ & 0.67 & 0.58 & 0.49 & 0.37 & 0.00 & 0.00 & 0.00 & 0.00 & 0.00 & 0.00 & 0.00 & 0.00 & 0.00\\
					&&$\sigma = 0.50$ & 0.57 & 0.52 & 0.46 & 0.42 & 0.37 & 0.33 & 0.29 & 0.22 & 0.00 & 0.00 & 0.00 & 0.00 & 0.00\\
					&&$\sigma = 1.00$ & 0.44 & 0.41 & 0.38 & 0.35 & 0.33 & 0.29 & 0.26 & 0.22 & 0.19 & 0.17 & 0.15 & 0.13 & 0.12\\
					\cline{2-16}
					&\parbox[t]{2mm}{\multirow{6}{*}{\rotatebox[origin=c]{90}{$\text{ADRE}_{\text{REG}}$}}}&
					$\sigma = 0.25, \lambda = 0.05$ & \textbf{0.70} & \textbf{0.64} & \textbf{0.57} & 0.45 & 0.00 & 0.00 & 0.00 & 0.00 & 0.00 & 0.00 & 0.00 & 0.00 & 0.00\\
					&&$\sigma = 0.25, \lambda = 0.10$ & 0.69 & 0.63 & 0.55 & 0.44 & 0.00 & 0.00 & 0.00 & 0.00 & 0.00 & 0.00 & 0.00 & 0.00 & 0.00\\
					&&$\sigma = 0.50, \lambda = 0.05$ & 0.61 & 0.56 & 0.51 & 0.46 & 0.40 & \textbf{0.36} & 0.30 & 0.25 & 0.00 & 0.00 & 0.00 & 0.00 & 0.00\\
					&&$\sigma = 0.50, \lambda = 0.10$ & 0.62 & 0.57 & 0.52 & \textbf{0.47} & \textbf{0.42} & 0.36 & 0.29 & 0.24 & 0.00 & 0.00 & 0.00 & 0.00 & 0.00\\
					&&$\sigma = 1.00, \lambda = 0.05$ & 0.48 & 0.45 & 0.41 & 0.37 & 0.36 & 0.32 & 0.30 & 0.26 & 0.23 & \textbf{0.22} & \textbf{0.18} & 0.15 & \textbf{0.14}\\
					&&$\sigma = 1.00, \lambda = 0.10$ & 0.47 & 0.44 & 0.40 & 0.38 & 0.36 & 0.33 & \textbf{0.30} & \textbf{0.27} & \textbf{0.24} & 0.20 & 0.18 & \textbf{0.16} & 0.13\\
					\hline
					
				\end{tabular}
			\end{center}	
		\end{table*}
	\begin{theorem}
		If T-CERTIFY does not abstain and returns a label $c$ with radius $r$, then with probability at least $1 - \alpha$, $g(\bx + \gamma) = c~\forall~\norm{\gamma}_2\leq r$, where $r$ is the returned radius in T-CERTIFY.
	\end{theorem}
	\begin{proof}[Proof Sketch] 
		By Proposition 1 we know that with probability at least $1 - \alpha$ $\underline{p_A} \leq p_1$ and  $\overline{p_B} \geq \max_{c\neq 1}p_c$ hold, where again suppose without loss of generality $R_1 = 1$. On this event, since T-CERTIFY does not abstain only when $\underline{p_A} > 0.5$, we know the top class is correctly predicted, i.e., $g(\bx) = R_1$. This completes the proof. 
	\end{proof}

	\section{Experiments}
	In this section, we evaluate the effectiveness of ADRE regularization and T-CERTIFY algorithm for randomized smoothed classifier. For the training procedure, we mainly compare with the basic single Gaussian perturbation augmented training, referred as Basic training\cite{cohen2019certified}, 
	 and SMOOTHADV-ersarial training \cite{salman2019provably}, as these two approaches achieve state-of-the-art $\ell_2$ robustness under standard (non-adversarial) and adversarial training scheme, respectively. For the certification algorithm, we mainly compare with CERTIFY. 
	
	To evaluate robustness, we focus on the approximate certified accuracy at radius $r$, defined as the fraction of samples which are classified correctly by the certification algorithm along with a certified radius being at least $r$. When comparing ADRE with other training methods, for direct comparison we only apply CERTIFY for robustness certification with significance level $\alpha = 0.001$ and number of samples $n_0 = 100, n = 100,000$. This means that we use $100$ Monte Carlo samples to predict the output of smoothed classifier, and $100,000$ to calculate a lower bound of certified radius for each sample that will hold with probability being at least $99.9\%$. 
	Note that the approximate certified accuracy is not equivalent to the lower bound of the true accuracy that holds with probability at least $1 - \alpha$ over the randomness of the CERTIFY algorithm, but the difference is negligible when $\alpha$ is small. We refer the reader to \cite{cohen2019certified} for details. For T-CERTIFY, we search over $\alpha = 0.1,0.2,\dotso,1.0$, where at $\alpha = 1.0$ it returns the same certified radius as in CERTIFY.

	\begin{figure*}[ht]
		\centering
		\subfigure[]{\includegraphics[width=60mm,height = 30mm]{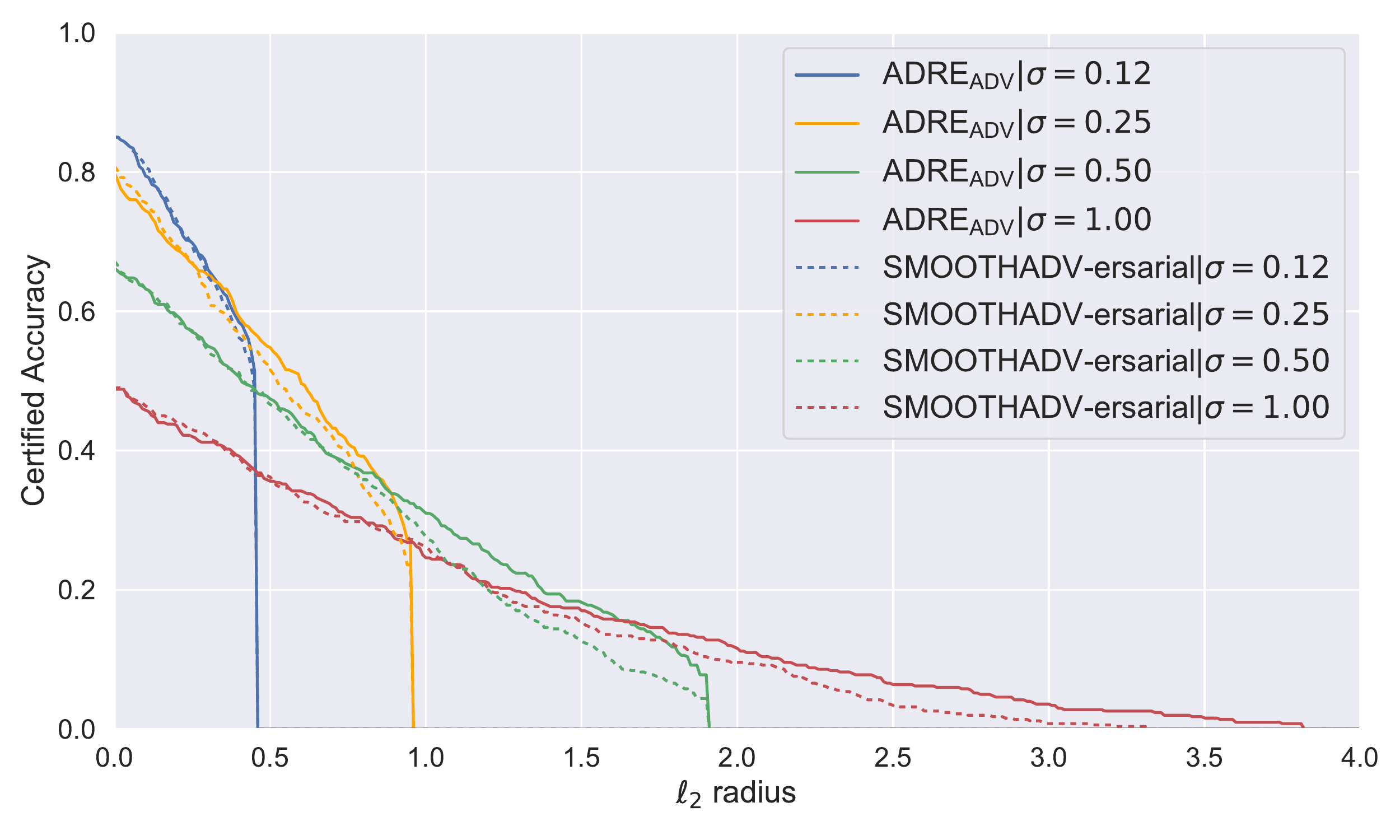} }
		\subfigure[]{\includegraphics[width=60mm,height = 30mm]{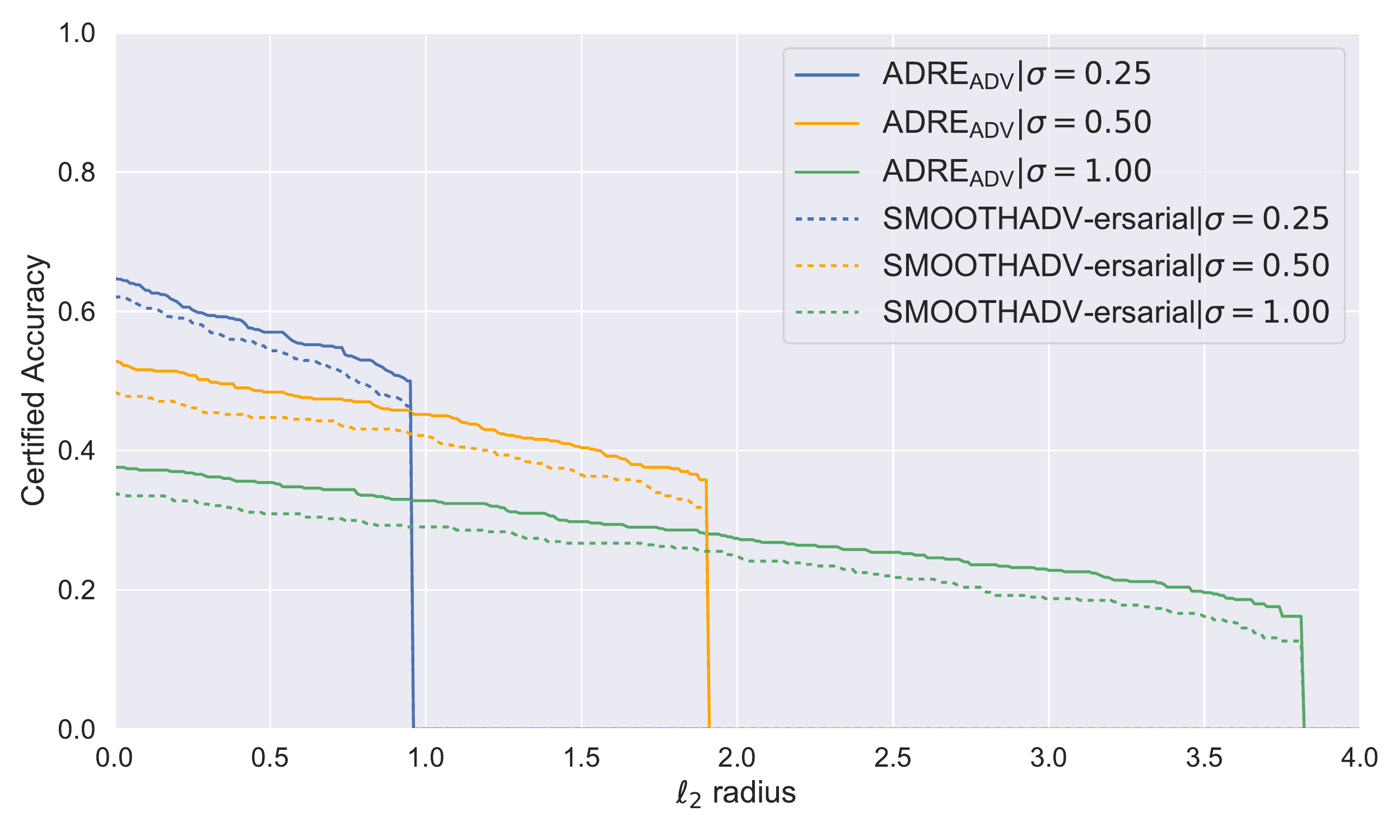}}
		\caption{Certified accuracy of smoothed classifier trained with $\text{ADRE}_{\text{ADV}}$ (solid line) vs SMOOTHADV-ersarial (dashed line) on (a) CIFAR-10 and (b) ImageNet. 
		}\label{plot-adv}
	\end{figure*}
	
	\begin{figure*}[ht]
		\centering
		\subfigure[]{\includegraphics[width=60mm,height = 30mm]{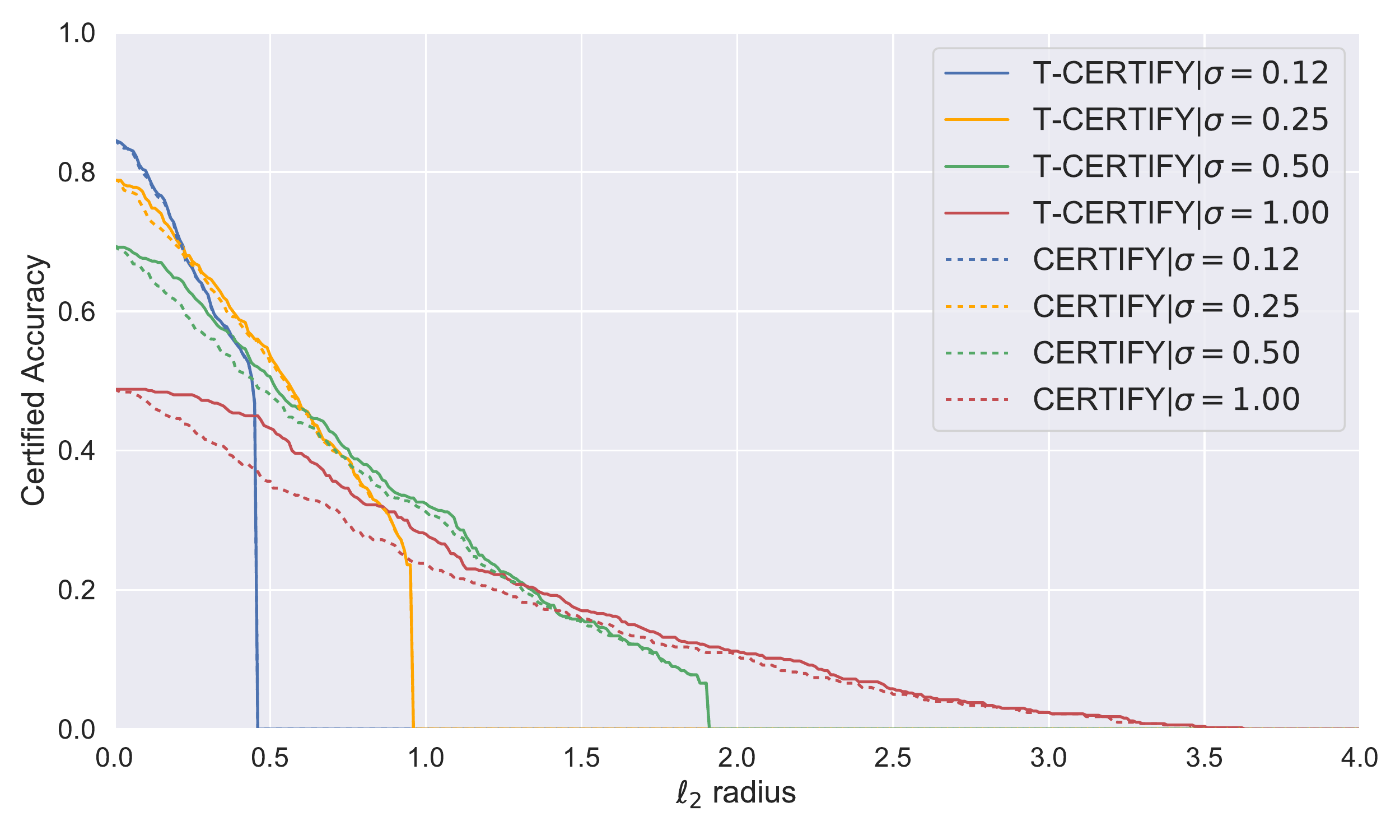} }
		\subfigure[]{\includegraphics[width=60mm,height = 30mm]{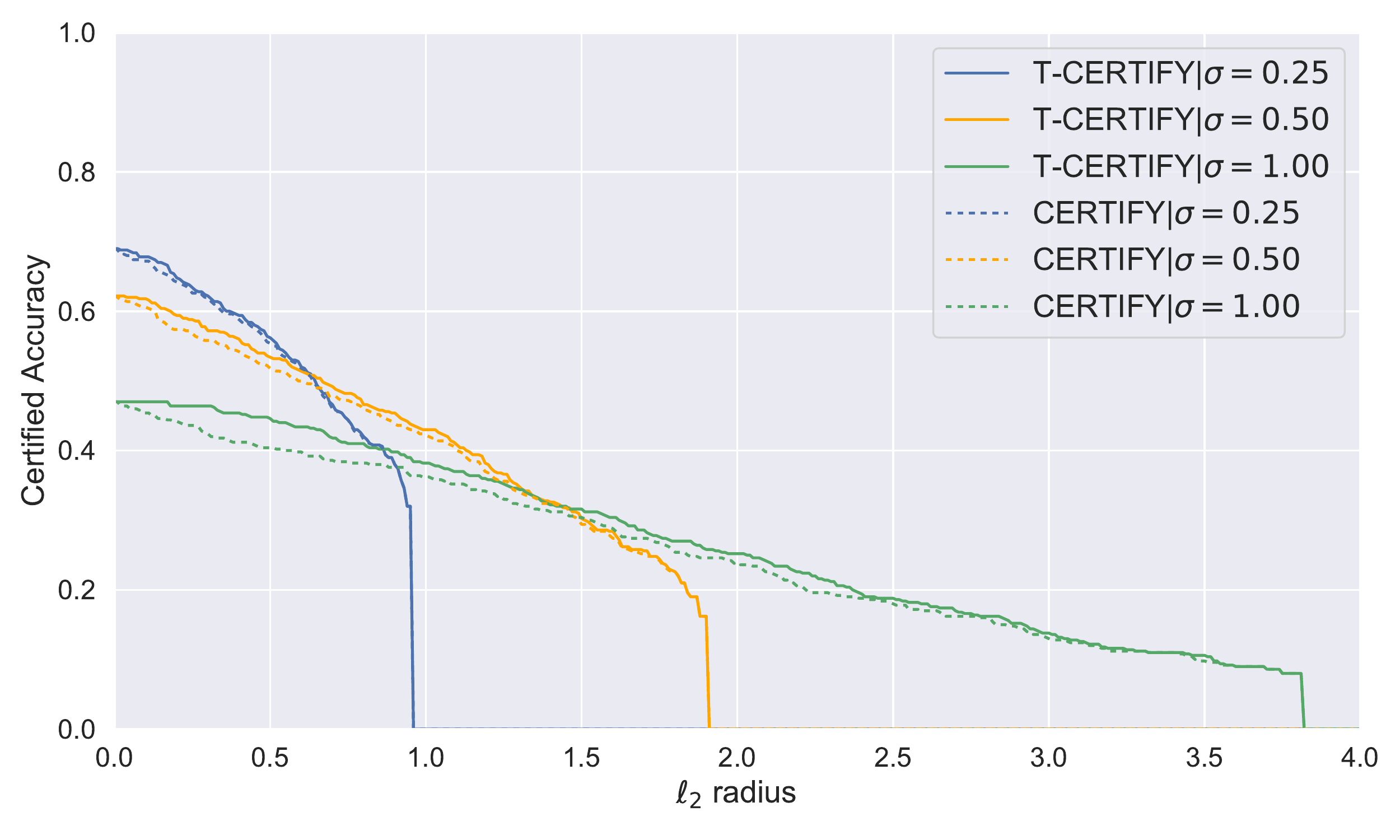}}
		\caption{Certified accuracy of one representative base classifier certified by our T-CERTIFY (solid line) vs CERTIFY (dashed line) on (a) CIFAR-10 and (b) ImageNet }\label{plot-certify}
	\end{figure*}
	We firstly assess the performance of $\text{ADRE}_{\text{REG}}$ and $\text{ADRE}_{\text{ADV}}$ training. We run experiments on CIFAR-10 \cite{krizhevsky2009learning} and ImageNet \cite{deng2009imagenet} datasets. 
	Consistent to compared work, we employ a 110-layer residual network and ResNet-50 as the base classifier for CIFAR-10 and ImageNet, respectively. For adversarial training, we used a constant step size $\alpha$ = $2\epsilon/M$ with $M$ being the number of attack iterations, and $\epsilon$ being the $\ell_2$ attack radius. On CIFAR-10, we trained the classifier using SGD on a single NVIDIA Tesla V100 GPU. We used a batch size of 400 with initial learning rate 0.1 which drops by a factor of 10 every 50 epochs, in the total 150 epochs. 
	On ImageNet, we trained the classifier on 4 NVIDIA Tesla V100 GPU using synchronous SGD with batch size 256 when $k = 1$ and 64 when $k = 4$, where $k$ is the number of Gaussian perturbations for plug-in estimates. We also used momentum (0.875), weight decay (1/32768), label smoothing (0.1) and cosine learning rate schedule for 50 epochs in total, where we set $0.1 \cdot epoch/8$ for warm-up and $0.05 \cdot (1 + cos(\pi \cdot epoch / (50 - 8)))$ afterwards. For both datasets, we trained the base classifier with random horizontal flips and random crops. Similar to compared work, the certified radii are with respect to original coordinate for direct comparison. We also added a centering layer as the first layer of the base classifier, which performed a channel-wise standardization, as implemented in \cite{salman2019provably}. 
	
	Table \ref{table-main} reports the approximate top-1 certified accuracy on CIFAR-10 and ImageNet comparing $\text{ADRE}_{REG}$ and Basic training. On  CIFAR-10, we train the base classifier with number of perturbations $k = 8$ and regularization $\lambda \in \{0.1,0.2,0.3\}$ for different magnitude of perturbations $\sigma \in \{0.12,0.25,0.50,1.00\}$. On ImageNet, we train with $k = 1$, $\lambda \in \{0.05, 0.1\}$ for $\sigma \in \{0.25,0.50,1.00\}$. 
	For a direct comparison, we slightly change the implementation of $\text{ADRE}_{REG}$ training on CIFAR-10. Specifically, instead of calculating $l_{per}^{(i)}$ following (\ref{plug-in-0}) as described in Algorithm \ref{algo-train}, in this experiment we only randomly sample a single perturbation for $l_{per}^{(i)}$, i.e., we let $l_{per}^{(i)} = l_{CE}(F(\bx_i + \bdelta_{ij};\btheta),y_i)$ for a random index $i \in [k]$. 
	By doing this, the only difference between $\text{ADRE}_{REG}$ and Basic training lies in ADRE regularization for both datasets. The results from Table \ref{table-main} suggests that ADRE regularization indeed improves the accuracy and robustness of smoothed classifier, where the certified robustness at zero radius is just the standard accuracy of the smoothed classifier. 
	In particular, with a proper hyper-parameter $\lambda$, for each perturbation $\sigma$ we can improve the certified radius up to 9\% on CIFAR-10 and 8\% on ImageNet without sacrificing the standard accuracy. 
	We point out that on ImageNet, there is little additional computation compared to Basic training.	We also run the original $\text{ADRE}_{REG}$ with $k = 8$ on CIFAR-10 and $k = 4$ for ImageNet. As is expected, we observe even stronger robustness at various radii when the base classifier is smoothed.
	
	In the next experiment, we compare SMOOTHADV-ersarial and the proposed $\text{ADRE}_{\text{ADV}}$ training. For demonstration, we focus on 2-step PGD adversarial training on CIFAR-10 with $k = 8$ and 1-step PGD on ImageNet with $k = 1$.  Figure \ref{plot-adv} plots the approximate certified accuracy of representative models on (a) CIFAR-10 and (b) ImageNet. Each solid line depicts the certified accuracy of a model trained by $\text{ADRE}_{\text{ADV}}$ and the dashed line depicts the certified accuracy of SMOOTHADV-ersarially trained model with the same $k$ and $\epsilon$, in which multiple Gaussian perturbation was applied for each training example on CIFAR-10. The results from Figure \ref{plot-adv} suggest that ADRE regularization is also useful under adversarial training scheme. 
	
	\subsubsection{Robustness Certification} 
	In this section, we evaluate the effectiveness of T-CERTIFY algorithm. We use the same $\alpha,n_0$ and $n$ as applied in CERTIFY. When certifying a given example $\bx$, we firstly generate a set of perturbations, and then use the same set of perturbed inputs to estimate $g(\bx)$ and calculate the certified radius. 
	This helps reduce uncertainty when comparing two approaches. 
		
	Figure \ref{plot-certify} depicts the certified accuracy from both approaches. We can observe that at each radius, T-CERTIFY yields higher certified robustness. In addition, we notice that the improvement gets more significant when $\sigma$ is larger. This is reasonable since with a larger perturbation, the confidence of the smoothed classifier may become lower. In this case, it becomes more important to estimate $\underline{p_A}$ and $\overline{p_B}$ separately in order to provide tighter lower bound for certified radius.

	\section{Conclusion}
	In this paper, we introduced a novel training procedure and certification algorithm for randomized smoothed classifier. We derived ADRE regularized risk and discussed how it can be implemented in both standard and iterative first-order adversarial training scheme. For certifying the (probabilistic) robustness of a smoothed classifier, we introduced T-CERTIFY to estimate lower bound for the $\ell_2$ robustness radius that will hold with high probability. We showed through experiments on CIFAR-10 and ImageNet datasets that ADRE regularization can improve the accuracy and $\ell_2$ robustness of the smoothed classifier, whose base classifier was trained under both standard and adversarial training scheme. We also demonstrated that T-CERTIFY can further improve the robustness guarantee based on the proposed regularization. 
	
	\bibliography{ref}
	\bibliographystyle{aaai}

\end{document}